\documentclass[11pt]{article}
\usepackage[english]{babel}
\usepackage{float}
\usepackage{amsmath}
\usepackage{bbold}
\usepackage{amssymb}
\usepackage{amsthm}
\usepackage{MnSymbol}
\usepackage{wasysym}
\usepackage{amsbsy}
\usepackage{parskip}
\usepackage{bbm}
\usepackage{cancel}
\usepackage{graphicx,color}
\usepackage[ruled]{algorithm2e}
\usepackage{mdframed}
\usepackage{natbib}
\usepackage{subcaption}
\newtheoremstyle{exampstyle}
{1em} 
{1em} 
{} 
{} 
{\bfseries} 
{.} 
{1.5em} 
{} 
\usepackage[letterpaper, margin=1in]{geometry}
\usepackage{personal}
\theoremstyle{exampstyle}
\newtheorem{definition}{Definition}[section]
\newtheorem{lemma}{Lemma}[section]

\newtheorem{corollary}{Corollary}[section]
\newtheorem{proposition}{Proposition}[section]

\newtheorem{assumption}{Assumption}[section]

\newcommand{\cQ}{\mathcal{Q}}
\newcommand{\ab}{\mathtt a^*}

\newcommand{\anv}{\mathtt{ a^*_{NV}}}
\newcommand{\alc}{\mathtt{ a^*_{LC}}}
\newcommand{\nX}{\tilde{X}_n}
\newcommand{\scKL}{\text{\sc KL}}
\newcommand{\argmin}{\text{argmin}}
\newcommand{\argmax}{\text{argmax}}
\newcommand{\qnv}{q^*(\theta|\nX)}
\newcommand{\qlc}{q^*_{a}(\theta|\nX)}
\newcommand{\qlcb}{q^*_{\bar a}(\theta|\nX)}
\title{Asymptotic Consistency of Loss-Calibrated Variational Bayes}
\author{Prateek Jaiswal{$^\star$}, Harsha Honnappa{$^\star$} and Vinayak A. Rao{$^\dag$} }
\date{$^\star$\{jaiswalp,honnappa\}@purdue.edu School of Industrial Engineering, Purdue University\\$^\dag$\{varao@purdue.edu\} Department of Statistics, Purdue University.}

\begin{document}

\maketitle


\begin{abstract}
    This paper establishes the asymptotic consistency of the {\it loss-calibrated variational Bayes} (LCVB) method. LCVB was proposed in~\cite{LaSiGh2011} as a method for approximately computing Bayesian posteriors in a `loss aware' manner. This methodology is also highly relevant in general data-driven decision-making contexts. Here, we not only establish the asymptotic consistency of the calibrated approximate posterior, but also the asymptotic consistency of decision rules. We also establish the asymptotic consistency of decision rules obtained from a `naive' variational Bayesian procedure.
    
\end{abstract}

\section{Introduction}
In this paper we establish the asymptotic consistency of {\it loss-calibrated variational Bayes} (LCVB). Consider a loss function $G(a,\theta) : (a,\theta) \mapsto G(a,\theta) \in \mathbb R$, where $a \in \sA \subset \bbR^s$ is a decision/design variable and $\theta \in \Theta \subset \mathbb R^d$ is a model parameter space. Given a set of observations $\nX = \{\xi_1,\ldots,\xi_n\}$ drawn from a distribution with unknown parameter $\theta_0$, $p(\nX|\theta_0)$, our goal is to compute the Bayes optimal decision rule
\begin{align}
  \mathtt a^*(\nX) := \arg\min_{a \in \sA} \bbE_{\pi}[G(a,\theta)] = \int_{\Theta} G(a,\theta) \pi(\theta|\nX) d\theta, \label{eq:opt_act}
\end{align}
where $\pi(\theta|\nX)$ is the posterior distribution. 
%
%
The latter results when a Bayesian decision-maker places a {\em prior} distribution $\pi(\theta)$ 
over the parameter space $\Theta$, capturing {\em a priori} information about 
$\theta$ such as location or spread. 
Given $\nX$, the prior and likelihood $p(\nX|\theta)$ together 
define a {\em posterior} distribution $\pi(\theta|\nX) \propto 
p(\nX|\theta)\pi(\theta) =: p(\theta,\nX)$, the conditional distribution over $\theta$ given 
observations. 
The posterior distribution represents uncertainty over the unknown 
parameter $\theta$, and contains all information required for 
further inferences or optimization. 

In general, under most realistic modeling 
assumptions, closed-form analytic expressions are unavailable for 
$\pi(\theta|\nX)$, making the subsequent integration and optimization problems 
intractable. In practice, therefore, one uses an approximation to the posterior in the integration in~\eqref{eq:opt_act}. It is easy to see that posterior computation can be expressed as a convex optimization problem:
\begin{align}
\label{eq:var_post}
	\min_{q(\cdot) \in \mathcal M}~ \text{KL}(q(\theta) \| \pi(\theta|\nX)) &= \text{KL}(q(\theta)\|p(\theta,\nX)) + \log p(\nX)\\
	\nonumber
	&=\text{KL}(q(\theta)\|\pi(\theta)) - \int_\Theta\log p(\nX|\theta) ~ q(\theta) d\theta + \log p(\nX)
\end{align}
where $\text{KL}$ is the Kullback-Leibler divergence and $\mathcal M$ is the space of all distributions that are absolutely continuous with respect to the posterior (or, equivalently, the prior). This problem can be immediately recognized as minimizing the `variational free energy'~\cite{neal1998view}. Variational Bayesian (VB) procedures~\cite{BlJo2006}, in standard form, restrict the optimization in~\eqref{eq:var_post} to a fixed subset $\mathcal Q \subset \mathcal M$. Here, we are interested in a generalized version of this procedure where the posterior computation is {\it calibrated} by the loss function $G(a,\theta)$ for each $a \in \mathcal A$:
\begin{align}~\label{eq:lcvb}
	\min_{q(\cdot) \in \mathcal Q}~&\text{KL}(q(\theta) \| G(a,\theta) \pi(\theta|\nX)) \\&=\text{KL}(q(\theta)\|p(\theta,\nX)) + \log p(\nX) - \int_{\Theta} \log G(a,\theta)~q(\theta) d\theta .
\end{align}
Observe that the set $\mathcal Q$ 
need not be convex. 
Consequently, this optimization problem is non-convex, in full generality, and practical algorithms for solving~\eqref{eq:lcvb} can only guarantee convergence to local minima. 
We leave the analysis of 
these optimization-related issues for future work, and focus instead on 
the global solution and its associated asymptotics. As we show later in Section~\ref{sec:lc1} that the optimal value of this loss-calibrated VB objective turns out to be a lower bound to $\log \mathbb E_{\pi}[G(a,\theta)]$, the logarithm of the loss in~\eqref{eq:opt_act}.

Loss-calibration was introduced in~\cite{LaSiGh2011} as a method for approximately computing a generalized Bayesian posterior, where the likelihood is re-weighted or calibrated by a loss function over the parameter space $\Theta$. 
As with most VB methods, theoretical properties of the approximations present largely unanswered questions. Recently, the theoretical properties of the variational Bayesian methods have been studied extensively in \citep{alquier2017concentration,abdellatif2018,jaiswal2019b,WaBl2017,yang2017alpha,ZhGa2017}. \cite{WaBl2017} established the asymptotic consistency of the VB approximate posterior and also proved a Bernstein-von Mises’ type result for the same. Whereas, the authors in \cite{ZhGa2017} studied the convergence rate of the VB approximate posterior. \cite{jaiswal2019b} presented a general framework for computing a risk-sensitive VB approximation and also studies the statistical performance of the inferred decision rules using these methods. Furthermore \cite{campbell2019universal,cherief2019gen,huggins2018practical,jaiswal2019a} studied theoretical properties of variational Bayesian methods defined using  Hellinger distance, Wasserstein distance, and R\'enyi divergence respectively instead of Kullback-Liebler (\scKL) divergence. In this paper, we study the asymptotic consistency of the loss-calibrated approximate posterior and the optimal decisions computed using the this approximate posterior, as the number of samples $n \to \infty$.

 More precisely, in Proposition~\ref{prop:3}, we show (for fixed $a \in \mathcal A$ and an appropriate subset of distributions $\mathcal Q$) that as $n \to \infty$ the optimizer of~\eqref{eq:lcvb} weakly converges to a Dirac delta distribution concentrated on the true parameter $\theta_0$ for almost every sequence generated from the true data generating process. This result shows that the posterior concentrates for any $a \in \mathcal A$. The reason for this is manifest: observe that $G(a,\theta) \pi(\theta|\nX) \propto \left(G(a,\theta) \pi(\theta) \right) p(\nX|\theta)$. Thus, the loss function can be seen as only changing the prior distribution in the posterior computation. 
 As the number of samples increases, we should anticipate that any calibration effect is diminished. Extending this result, in Proposition~\ref{prop:6} we show that the optimizers of the approximate decision making problem, computed using the loss calibrated VB posterior, are asymptotically consistent, in the sense that this set of optimizers will necessarily be included in the optimizers of the `true' objective $G(a,\theta_0)$. 
 
Finally, we illustrate our results on the so-called newsvendor problem, studied extensively in the operations research literature as a prototypical decision-making problem. In this problem, a newsvendor must decide on the number of newspapers to stack up before selling any over a given day. We operate under the assumption that the newsvendor can observe realizations of the demand, but does not know the precise data generation process. The goal is to find the optimal number of newspapers to stack that minimizes losses. We conduct numerical studies to show that both the  loss  calibrated and naive VB methods on this problem  are consistent.  

The remainder of the paper is organized as follows. In Section~\ref{sec:VB} we formally introduce decision-theoretic variational Bayesian methods. In Section~\ref{sec:lca} we prove that the LCVB approximate posterior is asymptotically consistent. We build on this result and prove the consistency of the optimal decisions, using both the LCVB and NVB methods, in Section~\ref{sec:asymptote}. Finally, we present our numerical results in Section~\ref{sec:inference}.

\section{Decision-theoretic Variational Bayes}\label{sec:VB}
\subsection{The Naive Variational Bayes (NVB) Algorithm~}
The idea behind standard VB is to approximate the intractable posterior
$\pi(\theta|\nX)$ with an element $q^*(\theta)$ of a simpler class of
distributions $\cQ$ known as \textit{variational family}. 
Popular examples of $\cQ$ include the family of Gaussian distributions, or the family of factorized `mean-field' distributions that discard correlations between components of $\theta$. 
A natural caveat to the choice of $\cQ$ is that these distributions should be absolutely continuous with respect to the posterior (or equivalently, the prior). 
The variational solution $q^*$ is the element of $\cQ$ that is `closest' to $\pi(\theta|\nX)$ in the sense of the Kullback-Leibler (KL) divergence:
\begin{align}~\label{eq:vb_opt}
	\min_{q(\theta) \in \cQ}~ \text{KL}(q(\theta) \| \pi(\theta|\nX)) &= \text{KL}(q\|p(\theta,\nX)) + \log p(\nX)\\
	\nonumber
	&=\text{KL}(q(\theta)\|\pi(\theta)) - \int_\Theta\log p(\nX|\theta) ~ q(\theta) d\theta + \log p(\nX).
\end{align}
VB approaches allow practitioners 
to bring tools from optimization to the challenging problem of Bayesian 
inference, with expectation-maximization~\citep{neal1998view} and  
gradient-based~\citep{kingma2013auto} methods being used to minimize 
equation~\eqref{eq:vb_opt}. 
Note that this optimization problem is non-convex, since the constraint set $\cQ$ is non-convex in general. 
Also, observe that the objective  $\scKL(q(\theta\| \pi(\theta|\nX)  ))$ in~\eqref{eq:vb_opt} 
only requires the knowledge of posterior distribution $\pi(\theta|\nX)$ up to the proportionality constant, since the normalizing term $\log p(\nX)$ does not depend on $q$.     


The natural variational approximation to the optimization in~\eqref{eq:opt_act} is to calculate the variational approximate expected
posterior loss of taking an action $a$, and then perform the following optimization
\begin{align}
   \anv(\nX) := \text{argmin } \bbE_{\qnv}[G(a,\theta)]. \label{eq:seo}
\end{align}
We call this the \textit{naive variational Bayes} (NVB) decision rule. This
algorithm involves two optimization steps in sequence, separating the 
approximation of the posterior in~\eqref{eq:vb_opt}
from the decision optimization~\eqref{eq:seo}. {This sequential
procedure, in general, involves a loss in performance
compared to~\eqref{eq:opt_act}.} This creates the desideratum for a calibrated approach that takes the loss function into consideration in computing an appropriate posterior.

\subsection{Loss-Calibrated Variational Bayes (LCVB) Algorithm~}\label{sec:lc1}
A more sophisticated approach is to jointly optimize $q$ and $a$; one 
would expect this to outperform the naive two-stage NVB
algorithm. 
Assuming that the objective $\inf_{a,\theta} G(a,\theta) > 0$, a {loss-calibrated} lower bound can be derived by applying Jensen's inequality to the logarithm of the objective in~\eqref{eq:opt_act}, obtaining
\begin{align}
\nonumber
\log \mathbb{E}_{\pi(\theta| \nX)}[G(a,\theta)]
 &=  \log \int_{\Theta} \frac{q(\theta)}{q(\theta)} G(a,\theta)
   \pi(\theta| \nX) d\theta 
\geq - \int_{\Theta} q(\theta)\log \frac{q(\theta)}{G(a,\theta)
  \pi(\theta| \nX)}d\theta~\forall a \in \sA. 
\end{align}
In particular, it can be seen that
\begin{align}
\nonumber
\min_{a \in \sA} \log \mathbb{E}_{\pi(\theta|\nX)}[G(a,\theta)] &\geq
                                                              \underset{a\in\sA}{\min}~\underset{q
                                                              \in
                                                              \mathcal{Q}}{\max}
                                                              -\scKL(q(\theta)||
                                                              \pi(\theta|
                                                              \nX)) 
                                                              \\
                                                              &+
                                                              \int_{\Theta}
                                                              \log
                                                              G(a,\theta)
                                                              q(\theta)
                                                              d\theta
                                                              =:
                                                              \mathcal
                                                              F(a,q; \nX).
                                                              \label{eq:reg-form}
\end{align}
We call~\eqref{eq:reg-form} the {\it loss-calibrated} (LC) variational 
objective. Since $\log(\cdot)$ is a monotone transformation, minimizing the logarithmic objective on the left hand side above is equivalent to~\eqref{eq:opt_act}. Now, for any given $a\in \sA$ we denote the (globally maximal) LCVB approximate posterior as
\begin{align} \label{eq:qlc}
\qlc := \argmax_{q\in \cQ} \mathcal F(a,q;\nX). 
 \end{align}
 If the risk function $G(a,\theta)$ is constant then $\qlc$, for every $a\in \sA$, is the same as $\qnv$. Akin to $\qnv$ in~\eqref{eq:vb_opt}, computing $\qlc$ only requires knowledge of the posterior distribution $\pi(\theta|\nX)$ up to a proportionality constant. The corresponding LCVB decision-rule is defined as
\begin{align}
  \vspace{-1in}
  \alc(\nX) := \underset{a \in \sA}{\arg\min} ~\underset{q
  	\in
  	\mathcal{Q}}{\max} ~ \mathcal F(a,q;\nX). \label{eq:lc_opt}
\end{align}

Observe that the lower bound
achieves the log posterior value precisely for $q$ such that
$\frac{q(\theta)}{G(a,\theta)}$ is proportional to the posterior $\pi(\theta | \nX)$. Furthermore,~\eqref{eq:reg-form} shows that the maximization
in the lower bound computes a `regularized' approximate
posterior. Regularized Bayesian inference~\citep{zhu2014bayesian} views
posterior computation as a variational inference problem with
constraints on the posterior space represented as bounds on certain
expectations with respect to the approximate posterior. The
loss-calibrated VB methodology can be viewed as a regularized Bayesian
inference procedure where the regularization constraints are imposed
through the logarithmic risk term $\int_{\Theta}\log G(a,\theta)
q(\theta) d\theta$. Observe, however, that our setting also involves a minimization over the decisions 
(which does not exist in the regularized Bayesian inference procedure). 


\section{Consistency of the LCVB Approximate Posterior}~\label{sec:lca}
Recall the definition of the LCVB approximate posterior $\qlc$ in~\eqref{eq:qlc} for any $a\in \sA$. In this section, we show regularity conditions on the prior distribution, the risk function, the likelihood model, and the  variational family, under which $\qlc$, for any $a\in \sA$,  converges weakly to a Dirac-delta distribution at the true parameter $\theta_0$.  
We first assume that the prior distribution satisfies
\begin{assumption}~\label{assume:2}
    The prior density function $\pi(\theta)$ is continuous with non-zero
    measure in the neighborhood of the true parameter
    $\theta_0$ and it is bounded by a positive constant $M$, that is $\pi(\theta) < M, \forall \theta \in \Theta$.
\end{assumption}
    The prior distribution with bounded density can be chosen from a large class of distribution, like the exponential-family distributions. The first condition, that the prior has positive density at $\theta_0$ is a common assumption in Bayesian consistency analysis, otherwise  the posterior will not have any measure in the ball around the  true parameter $\theta_0$. 
    

We also assume the expected loss, or risk function, $G(a,\theta)$ satisfies the following
\begin{assumption}~\label{assume:3}
    The risk function $G(a,\theta)$ is
    \begin{enumerate} 
        \item continuous in decision variable $a \in \sA$ and locally Lipschitz in the parameter $\theta \in \Theta$, that is $\forall a\in \sA$ and for every compact set $\mathcal{C}\subset{\Theta}$
    \[ |G(a,\theta)- G(a,\theta_0)|\leq L_{\mathcal{C}}(a) \|\theta-\theta_0\|, \]
    such that $L_{\mathcal{C}}(a)< \infty$ is the Lipschitz constant for any $ \theta \in \mathcal{C}$.
    \item uniformly integrable in $\theta$ with respect
    to any $q \in \cQ$ and for any $a \in \sA$.
    \end{enumerate}
\end{assumption}
 In order to analyze the consistency of the decisions in
this case, we make a further assumption on the log-likelihood
function (which follows \cite{WaBl2017}):


\begin{assumption}\label{assume:4}
  {The likelihood satisfies the \emph{local asymptotic normality} (LAN) condition. 
    In particular, fix $\theta \in \Theta$. The sequence of log-likelihood functions $\{\log  P^n(\theta) \}$ (where $\log  P^n(\theta)$ $ = \sum_{i=1}^n \log p(x_i|\theta)$) satisfies  the LAN condition, if there exist matrices $r_n$ and $I(\theta)$, and random vectors $\{ \Delta_{n,\theta} \}$ such that $\D_{n,\theta} \Rightarrow \sN(0,I(\theta)^{-1})$ as $n \to \infty$, and for every compact set $K \subset \mathbb{R}^d$ 
        \begin{equation*}
        \sup_{h \in K} \left| \log {P^n(\theta + r_n^{-1} h)} - \log {P^n(\theta)} - h^T I(\theta)
        \Delta_{n,\theta} + \frac{1}{2} h^T I(\theta)h \right| \xrightarrow{P_0} 0 \ \text{as $n  \to \infty$  }.
        \end{equation*}}
\end{assumption}
This LAN condition is typical in asymptotic analyses, holding for a wide 
variety of models and allowing the likelihood to be asymptotically 
approximated by a scaled Gaussian centered around 
$\theta_0$~\citep{vdV00}. We use $\Delta_{n,\theta} = \sqrt{n}(\hat \theta_n - \theta_0)$ in the proofs of our results, where $\hat \theta_n $ is the maximum likelihood estimate of $\theta_0$.

Next, we define the rate of convergence of a sequence of distributions to a Dirac delta distribution.
\begin{definition}[Rate of convergence]\label{def:roc}
    A sequence of distributions $\{q_n(\theta) \}$ converges weakly to $\delta_{\theta_1}$, $\forall \theta_1 \in\Theta$ at the rate of $\gamma_n$ if 
    \begin{enumerate}
        \item[(1)] the  sequence of means $ \{\check \theta_n := \int \theta q_n(\theta) d\theta \}$ converges to $\theta_1$ as $n\to \infty$, and 
        \item[(2)] the variance of $\{q_n(\theta) \}$ satisfies
        \[E_{q_n(\theta)}[\|\theta - \check \theta_n\|^2] = O\left (\frac{1}{\gamma_n^2} \right).\]
    \end{enumerate}
\end{definition}

We also define rescaled density functions as follows.
\vspace{.5em}
\begin{definition}[Rescaled density]
    For a random variable $\xi$ distributed as $d(\xi)$ with expectation $\tilde{\xi}$, for any sequence of matrices $\{t_n\}$, the density of the rescaled random variable $\mu := t_n (\xi - \tilde{\xi})$ is
    \[\check{d}_n(\mu) = |det(t_n^{-1})| d(t_n^{-1} \mu + \tilde{\xi}), \] where $det(\cdot)$ represents the determinant of the matrix.
\end{definition}\label{def:rescale}
Next, we place a restriction on the variational 
family $\cQ$:
\begin{assumption}~\label{assume:5}
    \begin{enumerate}
   \item The variational family  $\cQ$ must contain distributions that are absolutely continuous with respect  to the posterior distribution $\pi(\theta|\nX)$.
   \item There exists a sequence of distributions $\{q_n(\theta)\}$ in the variational family $\cQ$ that converges to a Dirac delta distribution $\delta_{\theta_0}$ at the rate of $\sqrt{n}$ and with mean $\int \theta q_n(\theta) d\theta = \hat \theta_n$, the maximum likelihood estimate.
   \item The  differential entropy of the rescaled density of such sequence of distributions is positive and finite.
    \end{enumerate}
\end{assumption}
The first condition is necessary, since the $\scKL$ divergence in~\eqref{eq:vb_opt}  and~\eqref{eq:reg-form} is undefined for any distribution $q \in \cQ$, that is not absolutely continuous with respect to the  posterior distribution. 
The Bernstein von-Mises theorem shows that under mild regularity conditions, the posterior converges to a Dirac delta distribution at the true parameter $\theta_0$  at the rate  of $\sqrt{n}$, and the  second condition is just to ensure that the $\scKL$  divergence is well defined for all large enough $n$. 
This condition does not, by any means, imply that the LCVB  and NVB  approximate posterior converges  to Dirac delta distribution at the true parameter $\theta_0$ as $n\to \infty$.  

The primary result in this section shows that the loss-calibrated approximate posterior $\qlc$ for any $a\in \sA$
is consistent  and converges to the Dirac-delta distribution at $\theta_0$. We establish the frequentist consistency of
LCVB approximate posterior, extending and building on the results in~\cite{WaBl2017}.  

\begin{proposition}
    \label{prop:3}
    Fix $a \in \sA$. Then, under Assumptions~\ref{assume:2},~\ref{assume:3},~\ref{assume:4}, and~\ref{assume:5}
    \begin{align}
        \label{eq:12}
        \qlc \in \underset{q \in \cQ}{\arg \min}~\scKL\left(
        q(\theta) \bigg \| \frac{G(a,\theta) \pi(\theta | 
            \nX)}{\int_{\Theta} G(a,\theta) \pi(\theta |
            \nX) d\theta}
        \right) \Rightarrow \delta_{\theta_0} ~P_{0}-\text{a.s. as}~n
        \to \infty.
    \end{align}
\end{proposition}
Some comments are in order for this result. Recall that loss-calibration of the
posterior distribution `weights' it by the risk of taking decision $a$,
$G(a,\theta)$. The optimization then finds the closest density
functions in the family $\cQ$ to this re-weighted posterior
distribution. The posterior re-weighting has the effect of `directing'
the VB optimization to the most informative regions of the parameter
sample space for the decision problem of interest. 
However, $G(a,\theta)$, which does not involve the data $\nX$,
effectively serves to change the prior distribution, and in the limit, modulo our regularity assumptions, the consistency of the approximate posterior is to be anticipated. 
The proof of the proposition is presented in the appendix.

Since for a constant risk function $G(a,\theta)$, the LCVB  approximate posterior $\qlc$ is same as NVB  approximate posterior $\qnv$, we recover the result obtained in Theorem 5(1) of~\cite{WaBl2017}. We rewrite the result as a corollary for completeness.

\begin{corollary}
    \label{corr:1}
    Under Assumptions~\ref{assume:2},~\ref{assume:4}, and~\ref{assume:5} 
    \begin{align}
    \label{eq:13}
    \qnv \in \underset{q \in \cQ}{\arg \min}~\scKL\left(
    q(\theta) \big \| \pi(\theta | 
        \nX)
    \right) \Rightarrow \delta_{\theta_0} ~P_{0}-\text{a.s. as}~n
    \to \infty.
    \end{align}
\end{corollary}

\section{Consistency of Decisions}~\label{sec:asymptote}
In this section we prove that the optimal decision estimated by the LCVB and NVB algorithms are consistent, in the sense
that for almost every infinite sequence, the optimal
decision rules $\anv$ and $\alc$ concentrate on the set of `true' 
optimizers 
\begin{align*}
  A^* := \arg\min_{a \in \mathcal A} G(a,\theta_0) = \int \ell(y,a) p(y|\theta_0) dy.
\end{align*}
For brevity, we define 
 \( 
 H_{q}(a) := \bbE_{q}[G(a,\theta)]\) for any distribution $q(\cdot)$ on $\theta$
~and~
\(
H_0(a) := G(a,\theta_0).
\) We place a typical, but relatively strong condition on the decision space that
\begin{assumption}~\label{assume:1}
    The decision space $\sA$ is compact.
\end{assumption}
Coupled with Assumption~\ref{assume:3}, this implies  that the risk 
function is uniformly bounded in the decision space. 

Now, suppose that the
true posterior $\pi(\theta|\nX)) $ is in the set $\cQ$. Then, the NVB approximate 
posterior  in~\eqref{eq:vb_opt} $\qnv$ equals $\pi(\theta|\nX)$, so that
the empirical decision-rule $\anv(\nX)$ coincides exactly
with the Bayes optimal decision rule $\ab(\nX)$.
The consistency
of the true posterior has been well-studied, and under
Assumption~\ref{assume:2} it is well known~\citep{Sc1965,Gh1997} that for any neighborhood
$U$ of the true parameter $\theta_0$
\begin{align}
\label{eq:4}
\pi(U | \nX)) \to 1 \quad~P_{0}-a.s.~\text{as}~n\to\infty,
\end{align}
where $P_{0}$ represents the true data-generation
distribution. Then, it follows from Assumption~\ref{assume:3} that
\begin{align}
\label{eq:7}
\sup_{a \in \sA} \left | H_{\pi(\theta|\nX)}(a) -
H_0(a) \right| \to 0~P_{0}-a.s.~\text{as}~n\to\infty.
\end{align}
It is straightforward to see that the limit result follows pointwise,
and the uniform convergence result follows from the uniform
boundedness of the loss functions. In the following section, we consider the typical case when the posterior $\pi(\theta|\nX))
\not \in \cQ$.

\subsection{Analysis of the NVB Decision Rule~}~\label{sec:vb}
 The first result of this section proves that the Bayes predictive loss,
$H_{q^*}(a)$ is (uniformly) asymptotically consistent as the
sample size grows. We relegate the proof to the appendix.
\begin{proposition}~\label{prop:1}
  Under the assumptions stated above, we have
  \begin{align}
    \label{eq:8}
    \sup_{a \in \sA} \left| H_{q^*}(a) - H_0(a) \right| \to 0~P_{0}-a.s.~\text{as}~n\to\infty.
  \end{align}
\end{proposition}
This proposition builds on~\cite[Theorem 5]{WaBl2017}, which shows that modulo Assumptions~\ref{assume:2},~\ref{assume:4}, and~\ref{assume:5}, 
the NVB approximate posterior distribution is asymptotically consistent(see Corollary~\ref{corr:1}). Using the consistency of $\qnv$ and Assumption~\ref{assume:3}(1), we first  establish the pointwise convergence of $H_{q^*}(a)$ to $H_0(a)$. Then we argue, using continuity of the risk function $G(a,\theta)$ in $a$  and the compactness of set $\sA$, that uniform convergence follows.    

A straightforward
corollary of Proposition~\ref{prop:1} implies that the optimal value
$V_{q^*} := \min_{a \in \sA} H_{q^*}(a)$ is
asymptotically consistent as well; the proof is in the appendix.

\begin{corollary}~\label{cor:1}
  Under
  Assumptions~\ref{assume:2},~\ref{assume:3},~\ref{assume:4},~\ref{assume:5}, and~\ref{assume:1}, with $V_0 := \min_{a \in \sA} H_0(a)$,
\(   
\left| V_{q^*} - V_0 \right| \to 0~P_{0}-a.s.~\text{as}~n \to \infty.
\)
\end{corollary}
The primary question of interest is the asymptotic consistency of the
optimal decision-rule $\anv$. Our main result proves that in the large
sample limit $\anv$ is a subset of the true optimal
decisions $A^*$ for almost all samples $\nX$.

\begin{proposition}~\label{prop:2}
  We have
  \begin{align}
    \label{eq:10}
    \left\{\anv(\nX) \subseteq A^*\right\}~P_{0}-a.s. 
    \text{ as } n \to \infty.
  \end{align}
\end{proposition}
We use the uniform convergence of $H_{q^*}(a)$  to $H_0(a)$ and argue that any decision which is not in the true optimal  decision set $A^*$, must not exist in NVB approximate optimal decision set $\anv(\nX)$  for large enough $n$. Once again, we relegate the proof to the appendix. 
Consequently, it follows that NVB optimal actions are asymptotically oracle regret
minimizing:

\begin{corollary}
  For any ${\bf a} \in A^*$ and $a^{*} \in \anv(\nX))$,~
\(    
H_0(a^*) \to H_0({\bf a}) ~P_{0}-a.s.~\text{as}~ n\to\infty.
\)
\end{corollary}
The result above is a straightforward implication of the continuity of $G(a,\theta_0)$ in $a$ and Proposition~\ref{prop:2} and therefore the proof is omitted. 

\subsection{Analysis of the LC decision rule}~\label{sec:lc} 
Now, recall from~\eqref{eq:reg-form} that the LC decision-rule is
\[
\alc(\nX) = \arg\min_{a \in \sA} ~\max_{q \in \cQ} -\scKL\left(q(\theta) \|
\pi(\theta|\nX) \right) + \int_\Theta q(\theta) \log G(a,\theta) d\theta.
\]
The next proposition shows that $\alc(\nX)$ is a subset of
the true optimal decision set $A^*$ in the large sample limit for almost
all sample sequences. We use similar ideas as used in Section~\ref{sec:vb}.

\begin{proposition}~\label{prop:6}
    We have
    \begin{align}
    \label{eq:14}
    \left\{ \alc(\nX) \subseteq A^* \right\}~P_{0}-a.s. ~\text{as}~n\to\infty.
    \end{align}
\end{proposition}
The proof is in the appendix. This result naturally implies that the
loss-calibrated VB optimal decisions are also oracle regret minimizing

\begin{corollary}
    For any $\mathbf a \in A^*$ and $a^{**} \in \alc(\nX)$,~
    \(
    H_0(a^{**}) \to H_0({\bf a}) ~\text{as}~n\to\infty.
    \)
\end{corollary}

\section{Numerical Example}\label{sec:inference}

In this section we present a simulation study of a canonical optimal decision making problem called the newsvendor problem. This problem has been extensively studied in the inventory management literature~\cite{bertsimas2005data,levi2015data,Sc1960}. Recall that the newsvendor loss function is defined as $$\ell( a,\xi) := h(  a-\xi)^+ + b(\xi-  a)^+$$
where $\xi  \in  [0,\infty)$ is the random demand, 
$a$ is the inventory or decision variable, and $h$ and $b$ are given positive constants. We assume that the  decision variable  $a$ take values in a compact decision space $\mathcal{A}$. 
We also assume that  the the random demand $\xi$ is exponentially distributed with 
unknown rate parameter $\theta_0 \in (0,\infty)$. The model risk can easily be derived as 
\begin{align}\label{eq:news-loss}
G(a,\theta)= \mathbb{E}_{P_{\theta}}[\ell(a,\xi)]= ha- \frac{h}{\theta} + (b+h) \frac{e^{-a \theta}}{\theta},
\end{align}
which is convex in $a$. Let $\nX := \{\xi_1, \xi_2 \ldots \xi_n\}$ be $n$ observations of the random demand, assumed to be independent and identically distributed. 
Next, we posit a non-conjugate inverse-gamma prior distribution over the rate parameter $\theta$ with shape and rate parameter $\alpha$ and $\beta$ respectively.
Finally, we run a simulation experiment using the newsvendor model described above  for a fix $\theta_0 = 0.68$, $ b=0.1, \alpha =1 \text{, and } \beta=4.1$. We use naive VB and LCVB algorithms to obtain the respective optimal decision $\anv$ and $\alc$ for 9 different values of $h \in \{0.001,0.002,\ldots 0.009\}$ and repeat the experiment over 1000 sample paths. In Figure~1, we plot the $50^{th}$ quantile of the $|\mathtt{a^*} - \mathtt{a_0^*} |$, where $\mathtt{a^*} \in \{\anv,\alc\}$ for this model. Observe that the optimality gap decreases quite rapidly for both the naive VB (left) and the loss-calibrated VB (right) methods.

\begin{figure}
	\centering
		\begin{subfigure}[b]{0.45\textwidth}
			\includegraphics[width=0.9\textwidth,height=0.75\textwidth]{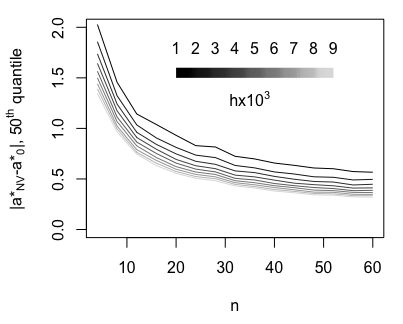}
		\end{subfigure}
		\begin{subfigure}[b]{0.45\textwidth}
			\includegraphics[width=0.9\textwidth,height=0.75\textwidth]{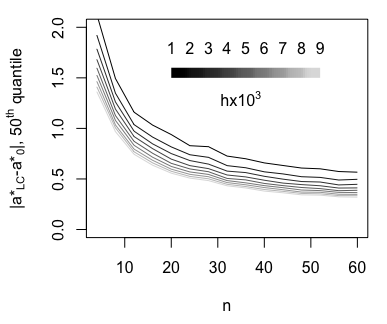}
		\end{subfigure}
{\caption{Optimality gap in decisions (the $50^{th}$ quantile over 1000 sample paths) against the number of samples ($n$) for $\anv$ (left) and $\alc$ (right).}}\label{fig1}
\end{figure}

\bibliographystyle{plain} 
\bibliography{refs.bib} 

\appendix
\section{Proof of Proposition~\ref{prop:3}}

\begin{lemma}\label{lem:MS}
     For any risk function $G(a,\theta)$ that satisfies Assumption~\ref{assume:3} and a given sequence of distributions $\{q_n(\theta)\}$ that converges weakly to any distribution $q(\theta)$ other than the Dirac-delta distribution at $\theta_0$, the $\scKL  \left(q_n(\theta) \bigg\| \frac{G(a,\theta) \pi(\theta)
         p(\nX|\mathbf \theta)} {\int_{\Theta} G(a,\theta) \pi(\theta) p(
         \nX |\theta) d\theta}\right)$ is undefined in the limit as $n \to \infty~P_0-a.s$.
    \end{lemma}
\begin{proof}
    Using the  definition of the posterior distribution $\pi(\theta|\nX) = \frac{ \pi(\theta)
        p(\nX|\mathbf \theta)} {\int_{\Theta}  \pi(\theta) p(
        \nX |\theta) d\theta} $, first observe that 
    \begin{align}
    \nonumber
        \scKL & \left(q_n(\theta) \bigg\| \frac{G(a,\theta) \pi(\theta)
            p(\nX|\mathbf \theta)} {\int_{\Theta} G(a,\theta) \pi(\theta) p(
            \nX |\theta) d\theta}\right) 
        \\
        \nonumber
        &= \scKL \left(q_n(\theta)\|\pi(\theta|\nX) \right) - \int_{\Theta} \log(G(a,\theta)) q_n(\theta) d\theta - \log \int_{\Theta} G(a,\theta) \pi(\theta|\nX) d\theta.
        \\
        &\geq \scKL \left(q_n(\theta)\|\pi(\theta|\nX) \right) - \int_{\Theta} G(a,\theta) q_n(\theta) d\theta -  \int_{\Theta} G(a,\theta) \pi(\theta|\nX) d\theta,
        \end{align}
     where the last inequality uses the fact that $\log x < x$.
    Now taking the $\liminf$ on either side, we have 
    \begin{align}
    \nonumber
   &\liminf_{n \to \infty} \scKL  \left(q_n(\theta) \bigg\| \frac{G(a,\theta) \pi(\theta)
        p(\nX|\mathbf \theta)} {\int_{\Theta} G(a,\theta) \pi(\theta) p(
        \nX |\theta) d\theta}\right) 
    \\
    &\geq \liminf_{n \to \infty} \scKL \left(q_n(\theta)\|\pi(\theta|\nX) \right) - \limsup_{n \to \infty}  \int_{\Theta} G(a,\theta) q_n(\theta) d\theta - \limsup_{n \to \infty} \int_{\Theta} G(a,\theta) \pi(\theta|\nX) d\theta.
    \label{eq:A1}
    \end{align}
    Recall that the posterior distribution $\pi(\theta|\nX)  $ converges weakly to $\delta_{\theta_0}$ $P_0-a.s$. Due to \cite[Theorem 16]{PosnerE.1975Rcsf} we know that $\scKL(q(\theta)\|p(\theta))$ is a lower semi-continuous function of the pair $(q(\theta),p(\theta))$ in the weak topology on the space of probability measures. Using lower semi-continuity, it follows that the first term in~\eqref{eq:A1} satisfies
        \begin{align} \label{eq:A2}
        \liminf_{n \to \infty} \scKL  \left(q_n(\theta) \| \pi(\theta|\nX)\right) > \scKL( q(\theta) \| \delta_{\theta_0}) = \infty, 
        \end{align}
        where the last equality is by definition of the $\scKL$ divergence, 
        since $q(\theta) \neq \delta_{\theta_0}$ (as $q_n(\theta)$ does not weakly converge to $\delta_{\theta_0}$) and therefore it is not absolutely continuous with respect to $\delta_{\theta_0}$. Since the last two terms are finite due to Assumption~\ref{assume:3}, we have shown  that for any sequence of distribution $\{q_n(\theta)\}$ that converges weakly to any distribution $q(\theta)\neq \delta_{\theta_0}$, the $\scKL  \left(q_n(\theta) \bigg\| \frac{G(a,\theta) \pi(\theta)
           p(\nX|\mathbf \theta)} {\int_{\Theta} G(a,\theta) \pi(\theta) p(
           \nX |\theta) d\theta}\right)$ diverges in the limit as $n \to \infty~P_0-a.s$.
    \end{proof}

\begin{lemma}\label{lemma:lem0}
    Let $\{ K_n \} \subseteq \Theta$ be a sequence of compact balls  such that for all $n\geq 1$, $\theta_0 \in K_n$ and  $K_n \to \Theta$ as $n \to \infty$. Then, under Assumption~\ref{assume:1} and for any $\delta>0$, the sequence of random variables $\left\{ \int_{\Theta \backslash K_n} \pi(\theta)G(a,\theta) \left(\frac{p(\nX|\theta)}{p(\nX|\theta_0)}\right) d\theta \right\}$ is of order $o_{P^n_{0}}(1)$; that is  \[ \lim_{n \to \infty} P_0^n \left( \int_{\Theta \backslash K_n} \pi(\theta) G(a,\theta) \left(\frac{p(\nX|\theta)}{p(\nX|\theta_0)}\right) d\theta >\delta \right) =0.  \]
\end{lemma} 
\begin{proof}
    Using Markov's inequality, it follows that,
    \begin{align}
    P_0^n \left( \int_{\Theta \backslash K_n} \pi(\theta) G(a,\theta) \left(\frac{p(\nX|\theta)}{p(\nX|\theta_0)}\right) d\theta >\delta \right) \leq \frac{1}{\delta} \bbE_{P_0^n}\left[ \int_{\Theta \backslash K_n} \pi(\theta) G(a,\theta) \left(\frac{p(\nX|\theta)}{p(\nX|\theta_0)}\right) d\theta \right].
    \end{align}
    Next, using Fubini's Theorem in the RHS above and then the fact  that $\bbE_{P_0^n}\left[  \left(\frac{p(\nX|\theta)}{p(\nX|\theta_0)}\right)  \right]\leq 1$, observe that 
    \begin{align}
    P_0^n \left( \int_{\Theta \backslash K_n} \pi(\theta) G(a,\theta) \left(\frac{p(\nX|\theta)}{p(\nX|\theta_0)}\right) d\theta >\delta \right) \leq \frac{1}{\delta}  \int_{\Theta \backslash K_n} \pi(\theta) G(a,\theta)  d\theta.
    \label{eq:eq21}
    \end{align}
    
    Since $\theta_0 \notin \Theta\backslash K_n$ for all $n \geq 1$ and $\Theta\backslash K_n \to \emptyset$ as $n \to \infty$, $\mathbb{1}_{\Theta \backslash K_n} G(a,\theta)$ is monotonic, and  therefore using the monotone convergence theorem,  $\int_{\Theta \backslash K_n} G(a,\theta) \pi(\theta) d\theta \to 0$ as $n \to \infty$. Hence, taking limits on either  side of~\eqref{eq:eq21}  the result follows.
\end{proof}

Next, we show that for fixed $a\in \sA$, the $\scKL$ divergence between the LC approximate posterior $\qlc$ and the rescaled posterior  $\frac{G(a,\theta) \pi(\theta)
    p(\nX|\mathbf \theta)} {\int_{\Theta} G(a,\theta) \pi(\theta) p(
    \nX |\theta) d\theta} $ is finite in the limit. Also, the following lemma uses similar proof techniques as used in~\cite{WaBl2017}.
\begin{lemma}\label{lemma:lem1}
Fix $a \in \sA$. Then, under Assumptions~\ref{assume:2},~\ref{assume:3},~\ref{assume:4}, and ~\ref{assume:5},
	\begin{align*}
		 \limsup_{n \to \infty}  \min_{q \in \cQ} \scKL  \left(q(\theta) \bigg\| \frac{G(a,\theta) \pi(\theta)
		 	p(\nX|\mathbf \theta)} {\int_{\Theta} G(a,\theta) \pi(\theta) p(
		 	\nX |\theta) d\theta}\right) < \infty . 
	\end{align*}
Furthermore, the LC variational posterior $\qlc$ can converge only at the rate of $\sqrt{n}$. 
\end{lemma}

\begin{proof}
Following Assumption~\ref{assume:5} there exists a  sequence  of distributions $\{q_n(\theta)\} \in \cQ$ that converges to $\delta_{\theta_0}$ at the rate of $\gamma_n=\sqrt{n}$.
Specifically, we consider the sequence where $q_n(\theta)$ has mean $\hat \theta_n$, the maximum likelihood estimate.
It suffices to show that for such  sequence $\{q_n(\theta)\} \subset \cQ$,
	\[\limsup_{n \to \infty} \scKL  \left(q_n(\theta) \bigg\| \frac{G(a,\theta) \pi(\theta)
		p(\nX|\mathbf \theta)} {\int_{\Theta} G(a,\theta) \pi(\theta) p(
		\nX |\theta) d\theta}\right) = \scKL  \left(q_n(\theta) \bigg \| \frac{G(a,\theta) \pi(\theta)
        \left(\frac{p(\nX|\theta)}{p(\nX|\theta_0)}\right)} {\int_{\Theta} G(a,\theta) \pi(\theta) \left(\frac{p(\nX|\theta)}{p(\nX|\theta_0)}\right) d\theta} \right) < \infty. \]

	For brevity let us denote $\scKL  \bigg(q_n(\theta) \bigg \| \frac{G(a,\theta) \pi(\theta)
		\left(\frac{p(\nX|\theta)}{p(\nX|\theta_0)}\right)} {\int_{\Theta} G(a,\theta) \pi(\theta) \left(\frac{p(\nX|\theta)}{p(\nX|\theta_0)}\right) d\theta} \bigg)$ as \scKL. 
        First, observe that for a compact set $K\subset \Theta$ containing the  true parameter $\theta_0$,  we have 
	\begin{align} 
	\nonumber
	\scKL 
	&= \int_{\Theta} q_n(\theta) \log (q_n(\theta)) d\theta -
	\int_{\Theta} q_n(\theta) \log ( G(a,\theta)
	\pi(\theta) ) d\theta - \int_{K} q_n(\theta)
	\log \left(\frac{p(\nX|\theta)}{p(\nX|\theta_0)}\right) d\theta
	\\
	\label{eq:eq0}
	& \quad \quad  - \int_{\Theta \backslash K} q_n(\theta)
	\log \left(\frac{p(\nX|\theta)}{p(\nX|\theta_0)}\right) d\theta + \log \left( \int_{\Theta} G(a,\theta) \pi(\theta) \left(\frac{p(\nX|\theta)}{p(\nX|\theta_0)}\right) d\theta \right).
	\end{align}
	
	Now we approximate $\int_{K} q_n(\theta) \log \left(\frac{p(\nX|\theta)}{p(\nX|\theta_0)}\right) d\theta$ using the LAN condition in Assumption~\ref{assume:4}. Let
    $\Delta_{n,\theta_0} := \sqrt{n}(\hat \theta_n - \theta_0)$, and reparameterizing the expression with $\theta =
    \theta_0 + n^{-1/2} h$ and denoting $K'$ as the reparameterized set $K$ we have 
	\begin{align}
	\nonumber
	\int_{K} & q_n(\theta) \log \left(\frac{p(\nX|\theta)}{p(\nX|\theta_0)}\right) d\theta  \\
    & = n^{-1/2} \int_{K'} q_n(\theta_0 + n^{-1/2}
	h) \log \left( \frac{ p(\nX|\theta_0 + n^{-1/2} h)}{ p(\nX|\theta_0 ) } \right) dh \\
	\nonumber
	&= n^{-1/2} \int_{K'} q_n(\theta_0 + n^{-1/2} h)
	\bigg ( h
	I(\theta_0)\Delta_{n,\theta_0} -	\frac{1}{2}h^2I(\theta_0)+
	o_{P^n_{0}}(1) \bigg) dh \\
	\nonumber
	&= \left( o_{P^n_{0}}(1) \right) \int_{K} q_n(\theta)  d\theta + \int_{K} q_n(\theta) \bigg (
	\sqrt{n}(\theta - \theta_0) I(\theta_0)\Delta_{n,\theta_0} -
	\frac{1}{2} n (\theta - \theta_0)^2 I(\theta_0) \bigg) d\theta  \\
	\label{eq:eq01}
	&= \left(  	\frac{1}{2} n I(\theta_0) (\hat \theta_n - \theta_0)^2 + o_{P^n_{0}}(1) \right) \int_{K} q_n(\theta)  d\theta  - \int_{K}	\frac{1}{2}  n I(\theta_0)q_n(\theta)  (\theta - \hat \theta_n)^2  d\theta.
	\end{align}

Now consider the last term in~\eqref{eq:eq0}. Let $\{ K_n \} \subseteq \Theta$ be a compact sequence of balls  such that for all $n\geq 1$, $\theta_0 \in K_n$ and  $K_n \to \Theta$ as $n \to \infty$.	
Next, using the same re-parametrization 
we obtain, \
\begin{align}
\int_{K_n} G(a,\theta) \pi(\theta) \left(\frac{p(\nX|\theta)}{p(\nX|\theta_0)}\right) d\theta &= e^{o_{P^n_{0}}(1)} e^{\frac{nI(\theta_0)}{2}(\hat \theta_n-\theta_0)^2 } \  \int_{K_n} G(a,\theta) \pi(\theta) e^{-\frac{nI(\theta_0)}{2}(\theta- \hat \theta_n)^2 } d\theta.
\label{eq:eq02a}
\end{align}	
Now, Lemma~\ref{lemma:lem0} implies that
\begin{align} 
\int_{\Theta \backslash K_n} G(a,\theta) \pi(\theta) \left(\frac{p(\nX|\theta)}{p(\nX|\theta_0)}\right) d\theta  = o_{P^n_{0}}(1).
\label{eq:eq03}
\end{align}
Hence, by the results in~\eqref{eq:eq02a} and~\eqref{eq:eq03}, the last term in~\eqref{eq:eq0} satisfies
\begin{align}
\nonumber
	\log& \left( \int_{\Theta} G(a,\theta) \pi(\theta) \left(\frac{p(\nX|\theta)}{p(\nX|\theta_0)}\right) d\theta \right) 
    \\
    \nonumber
    =& \log \left( \int_{K_n} G(a,\theta) \pi(\theta) \left(\frac{p(\nX|\theta)}{p(\nX|\theta_0)}\right) d\theta + \int_{\Theta \backslash K_n} G(a,\theta) \pi(\theta) \left(\frac{p(\nX|\theta)}{p(\nX|\theta_0)}\right) d\theta\right)  
    \\
    \sim & \frac{nI(\theta_0)}{2}(\hat \theta_n-\theta_0) + \log \int_{K_n} G(a,\theta) \pi(\theta) e^{-\frac{nI(\theta_0)}{2}(\theta- \hat \theta_n)^2 } d\theta  + o_{P^n_{0}}(1),
	\label{eq:eq02}
\end{align}  
	where $a_n \sim b_n$ implies that  $\lim_{n \to \infty} \frac{a_n}{b_n}=1$.
	Now, by substituting \eqref{eq:eq01} and \eqref{eq:eq02} into \eqref{eq:eq0} we obtain,
	\begin{align*}
	\scKL  \sim \int_{\Theta} &q_n(\theta) \log q_n(\theta) d\theta - \int_{\Theta} q_n(\theta) \log ( G(a,\theta) \pi(\theta) ) d\theta + \left(  	\frac{1}{2} n I(\theta_0) (\hat \theta_n - \theta_0)^2 + o_{P^n_{0}}(1) \right) \left[1- \int_{K} q_n(\theta)  d\theta \right] \\ 
	& + \log \int_{K_n} G(a,\theta) \pi(\theta) e^{-\frac{nI(\theta_0)}{2}(\theta- \hat \theta_n)^2 } d\theta  + \frac{1}{2} n  I(\theta_0) \int_{K} ( \theta- \hat \theta_n )^2	q_n(\theta)  d\theta  + o_{P^n_{0}}(1).
	\end{align*}

Since, the $q_n(\theta) \Rightarrow $  $\delta_{\theta_0}$ as $n \to \infty$ and $\theta_0 \in K$, 
\[ \left(  	\frac{1}{2} n I(\theta_0) (\hat \theta_n - \theta_0)^2 + o_{P^n_{0}}(1) \right) \left[1- \int_{K} q_n(\theta)  d\theta \right] \sim o_{P^n_{0}}(1) ,\]
implying that,
\begin{align}
	\nonumber
	\scKL \sim \int_{\Theta} q_n(\theta) \log q_n(\theta) d\theta &- \int_{\Theta}    \log ( G(a,\theta) \pi(\theta) ) q_n(\theta) d\theta - \frac{1}{2} \log n + \frac{1}{2} \log \left( \frac{2\pi} { I(\theta_0)} \right) \\ 
	  +  \log \int_{K_n} (G(a,\theta)\pi(\theta))& \mathcal{N}(\theta;\hat \theta_n,(nI(\theta_0))^{-1}) d\theta  + \frac{1}{2}  n  I(\theta_0) \int_{K} (\theta- \hat \theta_n)^2	q_n(\theta)  d\theta   + o_{P^n_{0}}(1),
	\label{eq:eq04}
	\end{align}
where $\mathcal{N}(\theta;\hat \theta_n,(nI(\theta_0))^{-1})$ represents the Gaussian density function.	
	Since $q_n(\theta)$ has mean $\hat \theta_n$ and rate of convergence $\sqrt{n}$, then by a change of variable to $\mu = \sqrt{n} (\theta - \hat \theta_n)$
	\begin{align}
		\int_{\Theta} &q_n(\theta) \log q_n(\theta) d\theta = \frac{1}{\sqrt{n}}  \int q_n\left(\frac{\mu}{\sqrt{n} }+ \hat \theta_n \right) \log q_n\left(\frac{\mu}{\sqrt{n} }+ \hat \theta_n \right)  d \mu = \frac{1}{2} \log n +   \int \check q_{n}(\mu) \log \check q_{n}(\mu)  d \mu,
		\label{eq:eq05}
	\end{align}
    where $\check q_{n}(\mu)$ is the  rescaled density as defined in Definition~\ref{def:rescale}. Substituting~\eqref{eq:eq05} into~\eqref{eq:eq04}, we obtain
\begin{align}
\nonumber
	\scKL \sim  \int &\check q_{n}(\mu) \log \check q_{n}(\mu)  d \mu - \int_{\Theta}    \log ( G(a,\theta) \pi(\theta) ) q_n(\theta) d\theta + \frac{1}{2} \log \left( \frac{2\pi} {I(\theta_0)} \right) \\ 
	&  +  \log \int_{K_n} (G(a,\theta)\pi(\theta)) \mathcal{N}(\theta;\hat \theta_n,(nI(\theta_0))^{-1}) d\theta  + \frac{1}{2}  n  I(\theta_0) \int_{K} (\theta- \hat \theta_n)^2	q_n(\theta)  d\theta   + o_{P^n_{0}}(1)
    \label{eq:L1}
  \end{align}
 Since, $\frac{1}{2}  n  I(\theta_0) \int_{K} (\theta- \hat \theta_n)^2	q_n(\theta)  d\theta  \leq \frac{1}{2}  n  I(\theta_0) \int_{\Theta} (\theta- \hat \theta_n)^2	q_n(\theta)  d\theta  \leq \frac{1}{2}I(\theta_0)$, due to the specific choice of $q_n(\theta)$ with variance $O(n^{-1})$(see Definition~\ref{def:roc}), it follows from~\eqref{eq:L1} that for large $n$,
  \begin{align}
    \nonumber
    \scKL \lesssim \int &\check q_{n}(\mu) \log \check q_{n}(\mu)  d \mu - \int_{\Theta}    \log (  \pi(\theta) ) q_n(\theta) d\theta - \int_{\Theta}    \log ( G(a,\theta) ) q_n(\theta) d\theta + \frac{1}{2} \log \left( \frac{2\pi} {I(\theta_0)} \right) 
    \\ 
    &  +  \log \int_{K_n} (G(a,\theta)\pi(\theta)) \mathcal{N}(\theta;\hat \theta_n,(nI(\theta_0))^{-1}) d\theta  + \frac{1}{2} I(\theta_0) + o_{P^n_{0}}(1)
    \label{eq:L2}
	\end{align}
	 Now take limsup on either side of the above equation. Observe that the first term is  finite by Assumption~\ref{assume:5}. The second term is finite since the prior distribution is bounded due to Assumption~\ref{assume:2}. For the third term, since $\log  x \leq x,$
     \[ \int_{\Theta}    \log ( G(a,\theta) ) q_n(\theta) d\theta \leq \int_{\Theta}    ( G(a,\theta) ) q_n(\theta) d\theta , \] and the RHS above is  bounded by Assumption~\ref{assume:3}. Since, $\theta_0\in K_n \forall  n\geq 1$, the fifth term is bounded by Laplace's approximation,
     \[\int_{K_n} (G(a,\theta)\pi(\theta)) \mathcal{N}(\theta;\hat \theta_n,(nI(\theta_0))^{-1}) d\theta  \sim G(a,\theta_0)\pi(\theta_0). \] Since the remaining terms are finite it follows that,
     \begin{align*}
     \limsup_{n \to \infty}  \min_{q \in \cQ} \scKL  \left(q(\theta) \bigg\| \frac{G(a,\theta) \pi(\theta)
         p(\nX|\mathbf \theta)} {\int_{\Theta} G(a,\theta) \pi(\theta) p(
         \nX |\theta) d\theta}\right) < \infty . 
     \end{align*}
%
	\end{proof}

    \begin{proof}[
        {Proof of Proposition~\ref{prop:3}}]
       Recall from the Lemma~\ref{lem:MS} that for any risk function $G(a,\theta)$ that satisfies Assumption~\ref{assume:3} and for a given sequence of distributions $\{q_n(\theta)\}$ that converges weakly to any distribution $q(\theta)$ other than $ \delta_{\theta_0}$, $\scKL  \left(q_n(\theta) \bigg\| \frac{G(a,\theta) \pi(\theta)
           p(\nX|\mathbf \theta)} {\int_{\Theta} G(a,\theta) \pi(\theta) p(
           \nX |\theta) d\theta}\right)$ diverges as $n \to \infty~P_0 -a.s$. On the other hand, Lemma~\ref{lemma:lem1} shows that for any $a \in \sA$, 
       \begin{align}
       \nonumber
      & \limsup_{n \to \infty}  \min_{q \in \cQ} \scKL  \left(q(\theta) \bigg\| \frac{G(a,\theta) \pi(\theta)
           p(\nX|\mathbf \theta)} {\int_{\Theta} G(a,\theta) \pi(\theta) p(
           \nX |\theta) d\theta}\right) 
       \\
       &= \limsup_{n \to \infty}  \scKL  \left(\qlc \bigg\| \frac{G(a,\theta) \pi(\theta)
           p(\nX|\mathbf \theta)} {\int_{\Theta} G(a,\theta) \pi(\theta) p(
           \nX |\theta) d\theta}\right) < \infty. 
       \end{align}
Therefore, Lemma~\ref{lem:MS} and~\ref{lemma:lem1}  combined together imply that for any $a \in \sA$, and for any risk function $G(a,\theta)$ that satisfies Assumption~\ref{assume:3}, the LC approximate posterior must converge weakly to $\delta_{\theta_0}$ as  $n\to \infty~P_0-a.s$; that is 
\(\qlc \Rightarrow \delta_{\theta_0}~P_0-a.s.  \text{ as } n \to \infty.\)
        \end{proof}

\section{Proof of Proposition~\ref{prop:1}}
\begin{proof}
Fix $a \in \sA$. Observe that for any $\eta>0$.
\begin{align} 
\nonumber
&\left|\int_{\Theta}G(a,\theta) \qnv d\theta - G(a,\theta_0)\right|  
\\
\nonumber
&\leq \int_{\Theta}|G(a,\theta)- G(a,\theta_0)| \qnv d\theta. 
\\
& = \int_{\|\theta-\theta_0\|>\eta}|G(a,\theta)- G(a,\theta_0)| \qnv d\theta +  \int_{\|\theta-\theta_0\|\leq \eta}|G(a,\theta)- G(a,\theta_0)| \qnv d\theta.
\label{eq:eq1}
\end{align} 
 Next, recall the  fact that NV approximate  posterior is consistent from \cite{WaBl2017}, that is for every $\eta>0$
\begin{align} \lim_{n \to \infty} \int_{\|\theta-\theta_0\|>\eta} \qnv d\theta  = 0 ~ P_{0}-a.s. 
\label{eq:eq2}
\end{align} 

Using the above result and the monotone convergence theorem, the first term in~\eqref{eq:eq1} converges to zero.

By Assumption~\ref{assume:3} $G(a,\theta)$ is
locally Lipschitz continuous in $\theta$.
where $L_{\eta}(a)$ is the Lipschitz constant for the  compact  set $\{ \theta\in \Theta: \|\theta-\theta_0\|  \leq \eta  \}$. Using this fact in the  second term of~\eqref{eq:eq1}, we obtain
\begin{align} 
\nonumber
\int_{\|\theta-\theta_0\|\leq \eta}|G(a,\theta)- G(a,\theta_0)| \qnv d\theta   &\leq L_{\eta}(a)\int_{\|\theta-\theta_0\|\leq \eta}\|\theta- \theta_0\| \qnv d\theta
\\
& \leq L_{\eta}(a)\eta\int_{\|\theta-\theta_0\|\leq \eta} \qnv d\theta.
\label{eq:eq3}
\end{align} 
Now taking limit on either side of~\eqref{eq:eq3}, and  Using~\eqref{eq:eq2} again, the second term in~\eqref{eq:eq3} tends to $L_{\eta}(a)\eta$, but since $\eta$ can be arbitrarily small (and $L_{\eta}(a) \to 0 $ as $\eta \to 0$  ), we obtain
\[ \lim_{n \to \infty} \left|\int_{\Theta}G(a,\theta) \qnv d\theta - G(a,\theta_0)\right|   = 0 ~P_{0}-a.s~\forall a \in \sA.   \] 
It follows straightforwardly that $H_{q^*}(a):=
E_{\qnv}[G(a,\theta)]$ converges to $H_0(a)$ 
$P_{0}$-a.s. for any $a\in \mathcal{A}$. That is, 
\begin{align}
\nonumber
&\forall \ a \in \ \mathcal{A}, \ P_0\left(\left\{\omega \in \Omega: \underset{n \rightarrow \infty}{\lim} \left|H_{q^*(\theta|\nX(\omega))}(a)-H_0(a)\right|=0\right\}\right)= 1, ~\text{and}\\
\label{eq:prop12}
& \forall \ a \in \ \mathcal{A}, \ P_0\left(\left\{\omega \in \Omega: \forall \epsilon >0 \ \exists n\geq n(\omega,a,\epsilon) : \left|H_{q^*(\theta|\nX(\omega))}(a)-H_0(a)\right|<\epsilon\right\}\right)= 1.
\end{align}

Next, define $E:= \{ a: a \in Q \cap \mathcal{A}\}$, where $Q$ is set
of rationals on $\mathbb R$. Thus $E$ is a countable set containing
the rational numbers (say $\{a_i\}$) in $\mathcal{A}$. Define a set
$A_i=\{\omega \in \Omega: \forall \epsilon >0 \ \exists n \geq
n(\omega,a_i, \epsilon) : |H_{q^*(\theta|\nX(\omega))}(a_i)-H_0(a_i)|> \epsilon\}$. From~\eqref{eq:prop12} we know that $P_0(A_i)=0$. We also know that countable unions of sets of probability measure 0 are also of probability measure 0; that is, $P_0\left(\bigcup_{i\geq 1} A_i\right)\leq \sum_{i\geq 1} P_0\left(A_i \right)=0$. By definition of $A_i$, it follows that 
\begin{align*}
 \bigcap_{i\geq 1} A_i^C&= \left\{
\begin{array}{l}\omega \in \Omega: \forall
\epsilon > 0~\exists n\geq n(\omega,
\epsilon) \geq \text{max}_{i \geq 1}
\{n(\omega,a_i, \epsilon)\},\\
~\text{s.t.}~
|H_{q^*(\theta|\nX(\omega))}(a_i)-H_0(a_i)|< \epsilon, \
\forall \ \{a_i\}_{i \geq 1} \end{array}\right\}.
\end{align*}
This establishes the uniform convergence $\sup_{a_i \in
    E}|H_{q^*(\theta|\nX(\omega))}(a_i)-H_0(a_i)| \rightarrow 0$ $P_0-$a.s. 

By the continuity of $G(a,\theta)$ in Assumption~\ref{assume:3}, $H_{q^*}(a)$ and $H_0(a)$ are
continuous over the compact set $\mathcal{A}$, and by the Heine-Cantor
theorem, $H_{q^*}(a)$ and $H_0(a)$ are uniformly continuous on
the set $\mathcal{A}$. Since the set $E$ is dense in the compact set
$\mathcal {A}$, we can find a sequence $\{a_n\} \in E$ that converges to a
point $a \in \mathcal {A}$. Then, it follows that
$$|H_{q^*}(a)-H_0(a)| \leq |H_{q^*}(a)-H_{q^*}(a_n)|+|H_{q^*}(a_n)-H_0(a_n)|+|H_0(a_n)-H_0(a)|.$$

From the above inequality, we can conclude that $\sup_{a \in
    \sA}|H_{q^*}(a)-H_0(a)| \rightarrow 0$ $P_0-$a.s., by using
the uniform continuity of $H_{q^*}(a)$ and $H_0(a)$ on set $\mathcal{A}$
and the uniform convergence of $\sup_{a_n \in E}|H_q(a_n,X)-H_0(a_n)|$ to 0$~P_0-$a.s, thus completing the proof.
\end{proof}

\section{Proof of Corollary~\ref{cor:1}}
\begin{proof}
Let $a_q \in \anv(\nX)$ and $a_0 \in A^*$ then, by definition, $V_{q^*} = H_{q^*}(a_q)$ and $V_0 = H_0(a_0)$. Then, 
\begin{align}
V_{q^*}-V_0 &= [H_{q^*}(a_q)-H_0(a_0) ] \leq [H_{q^*}(a_0)-H_0(a_0)] \leq \underset{a \in \mathcal{A}}{\sup}   |H_{q^*}(a)-H_0(a)| \label{eq:16}.
\end{align}
On the other hand, observe that
\begin{align}
V_{q^*}-V_0 & \geq [H_{q^*}(a_q)-H_0(a_q)] \geq - |H_{q^*}(a_q)-H_0(a_q)| \geq - \underset{a \in \mathcal{A}}{\sup} |H_{q^*}(a)-H_0(a)|.\label{eq:17}
\end{align} 

Therefore from~\eqref{eq:16},~\eqref{eq:17}, and
Proposition~\ref{prop:1}, it follows that
\begin{align*}
\lim_{n \to \infty} |V_{q^*}-V_0| \leq \lim_{n \to \infty} \underset{a \in \mathcal{A}}{\sup}
|H_q(a,X)-H_0(a)| = 0~P_0-a.s,
\end{align*}
and the result follows.
\end{proof}
\section{Proof of Proposition~\ref{prop:2}}
\begin{proof}
Equivalently, we can show that $a \in \mathcal{A}\setminus A^*$ implies that
$a \not\in \anv(\nX)~P_0-$a.s. as $n\to\infty$. Fix $a \in \mathcal{A}\setminus A^*$, then
we have $H_0(a) > V_0$. Next define $\epsilon:=\underset{a \in
    \mathcal{A}\setminus A^*}{\inf} H_0(a) - V_0 $.  Using Proposition~\ref{prop:1}, there exists an $ n_0\geq 1$
such that $\forall n \geq n_0$, $|V_{q^*}-V_0| \leq \underset{a \in
    \mathcal{A}}{\sup} |H_{q^*}(a)-H_0(a)| < \frac{\epsilon}{2}$
$P_0-a.s.$ Therefore we have, $ V_{q^*} < V_0 + \frac{\epsilon}{2}$ for all $n\geq n_0$. We also know that for any $a \in \mathcal{A}\setminus A^*$  and $n \geq n_0$
\begin{align*}
\epsilon+ V_0 &=\underset{a \in \mathcal{A}\setminus A^*}{\inf} H_0(a)
\leq  H_0(a) < H_{q^*}(a) + \frac{\epsilon}{2},
\end{align*}
which implies that $H_{q^*}(a) > \frac{\epsilon}{2} + V_0$. Therefore for any $a\in \mathcal{A}\setminus A^*$ , $H_{q^*}(a) > V_0 +
\frac{\epsilon}{2}> V_{q^*}$ for all $n \geq n_0$. This implies that $a \not\in
\anv(\nX)$ for all $n \geq n_0~P_0-$a.s. and hence the proposition follows.
\end{proof}

\section{Proof of Proposition~\ref{prop:6}}
\begin{proof}
Fix  $\bar a \in \mathcal{A}$ and recall from~\eqref{eq:reg-form} that \[\mathcal F(a,q; \nX) = -\scKL (q(\theta)\| \pi(\theta|\nX)) + \int_{\Theta} \log G(a,\theta) q(\theta)  d\theta.\] Also recall that the LC approximate posterior $\qlcb$ converges weakly to a Dirac delta distribution at $\theta_0$ due to Propostion~\ref{prop:3}. 
It now follows that, due to Assumption~\ref{assume:3}~(2) on $G(a,\theta)$ and application of dominated convergence theorem
 \begin{align}
  \lim_{n \to \infty} \int_{\Theta} \log G(a,\theta) \qlcb  d\theta  = \log G(a,\theta_0) \ \forall a \in \sA \ P_0-a.s. 
  \label{eq:eq6}
  \end{align}

Now since the set $\sA$ is compact, logarithm function is continuous, and $G(a,\theta)$ is continuous in $\forall a\in \sA$, it follows using similar arguments as used in~Proposition~\ref{prop:1} that for any $\bar a \in \sA$,
\begin{align}
    \sup_{a \in \sA} \left| \int_{\Theta} \log G(a,\theta) \qlcb  d\theta  - \log G(a,\theta_0)   \right| \to 0 \ P_0-a.s.
\end{align}
Now again using similar arguments as in Proposition~\ref{prop:2} and monotonicity of logarithm function, we can show that the LC approximate decision rule for any $\bar a \in \sA$, that is $$\alc(\nX,{\bar a}) : = \argmin_{a \in \sA} \int_{\Theta} \log G(a,\theta) \qlcb  d\theta $$ is subset of the true decision set $A^*$ $P_0-a.s.$ as $n \to \infty$. Since the result is true for any $\bar a \in \sA$, it is true for any $a$ that lies in LC approximate decision set $\alc$ and therefore the proposition follows. 
\end{proof}

\end{document}